\documentclass[journal]{IEEEtran}

\usepackage[utf8]{inputenc} 
\usepackage[T1]{fontenc}    
\usepackage{hyperref}       
\usepackage{url}            
\usepackage{booktabs}       
\usepackage{amsfonts}       
\usepackage{microtype}      
\usepackage{graphicx}
\usepackage{xcolor}
\usepackage[cmex10]{amsmath}
\usepackage{amssymb}
\usepackage{amsthm}
\usepackage[tight,footnotesize]{subfigure}
\usepackage{caption}
\usepackage{booktabs}
\usepackage{comment}

\newcommand{\lfunc}{\mathcal{\bar{L}}}
\newcommand{\R}{\mathbf{R}}
\newcommand{\bR}{\mathbb{R}}
\newcommand{\bC}{\mathbb{C}}
\newcommand{\sF}{\mathcal{F}}
\newcommand{\sA}{\mathcal{A}}

\newtheorem{proposition}{Proposition}
\newtheorem{theorem}{Theorem}
\newtheorem{remark}{Remark}

\begin{document}
\title{The Loss Surface Of Deep Linear Networks Viewed Through The Algebraic Geometry Lens}

\author{Dhagash Mehta, Tianran Chen, Tingting Tang and Jonathan D. Hauenstein%
\thanks{D.~Mehta is with the Autonomous and Intelligent Systems Department, United Technologies Research Center,
East Hartford, CT, USA. e-mail: dhagashbmehta@gmail.com}%
\thanks{T.~Chen is with the Department Mathematics and Computer Science,
Auburn University at Montgomery, Montgomery, AL, USA
e-mail: ti@nranchen.org.}%
\thanks{T.~Tang and J.D.~Hauenstein are with Department of Applied and Computational Mathematics and Statistics, 
University of Notre Dame, Notre Dame, IN, USA. e-mail: \{ttang,hauenstein\}@nd.edu}
}

\maketitle

\begin{abstract}
    By using the viewpoint of modern computational algebraic geometry,
    we explore properties of the optimization landscapes of the deep linear neural network models. After clarifying on the various definitions of "flat" minima, we show that the geometrically flat minima, which are merely artifacts of residual continuous symmetries of the deep linear networks, can be straightforwardly removed by a generalized $L_2$ regularization. Then, we establish upper bounds on the number of isolated stationary 
    points of these networks with the help of algebraic geometry. Using these upper bounds and utilizing a numerical algebraic geometry method, we find \textit{all} stationary points for modest depth and matrix size. We show that in the presence of the non-zero regularization, deep linear networks indeed possess local minima which are not the global minima. Our computational results further clarify certain aspects of the loss surfaces of deep linear networks and provides novel insights.
\end{abstract}



\section{Introduction}
Advancement in both computational algorithms and computer hardware has led a surge in applied and
theoretical research activities for deep learning techniques. Though the applied side of the research
has been remarkably successful with applications in such areas as computer vision, natural language
processing, machine translation, object recognition, speech and audio recognition, stock market analysis,
bioinformatics, and drug analysis \cite{lecun2015deep,bengio2015deep}, a thorough theoretical understanding
of the techniques is yet to be achieved. 

One of the urgent theoretical issues that is of particular interest to the present work is the highly
non-convex nature of the underlying optimization problems that the techniques bring with them: the cost
function (also called the loss function) of a typical deep learning task, such as the mean squared error
between the observed data and predicted data from the deep network, is known to have numerous local minima. 
Finding a minimum which possesses a desired characteristic is usually a daunting task, especially for a
high-dimensional problem, and most of the times it turns out to be an NP hard problem
\cite{blum1988training}. Nonetheless, in practice, for a typical deep learning task, a reasonably good
minimization algorithm, such as a stochastic gradient descent (SGD) based method, converges to a minimum
that performs well. This observation, along with several empirical results
\cite{baldi1989neural,gori1992problem,yu1995local,saxe2013exact,NguyenLoss,goodfellow2014qualitatively,andoni2014learning,soudry2016no,2017arXiv171007406L}, 
has led to the belief that there is no bad minima in the loss functions of \textit{deep}
networks. In \cite{choromanska2014loss} (cf., \cite{kawaguchi2016deep,soudry2017exponentially}), the loss function of a typical dense feed-forward
neural network with rectified linear (ReLu) units was approximated by the Hamiltonion of a physics model
called the spherical $p$-spin model and analyzed using random matrix theory and statistical physics
techniques. It was concluded that for this approximate model, the number of minima and saddle points at
which the value of the loss function is beyond certain threshold vanishes as the number of hidden layers
increases (cf. \cite{sagun2014explorations}) supporting the ``no bad minima'' scenario, 
though the assumptions made to bring the deep network to the spin glass model were unrealistic. 

The specific characteristics of the minima that numerical minimization algorithms may be looking for may 
play a crucial role in determining if and why the 
algorithm finds them so efficiently \cite{mehta2018}. 
In the literature, the distance of a minimum from the global minimum has been the defining characteristic 
of the ``goodness'' of minima, i.e., if the difference between the loss function at the local 
minimum and that at the global minimum is within certain threshold, 
then the minimum is good enough for the task. 
There are recent examples of artificial neural networks with such suboptimal minima for deep nonlinear networks 
\cite{coetzee1997488,swirszcz2017local,BallardDMMSSW17,yun2018critical} 
(and \cite{sontag1989backpropagation} for neural networks without hidden layers). In \cite{wu2017towards}, 
good and bad minima are distinguished based not only in terms of the the performance of the network on the 
training data but also on the testing data, and it is empirically shown that the volume of basin of 
attraction of good minima dominates over that of bad minima (cf., \cite{mehta2018, kawaguchi2017generalization} for discussions on good and bad minima).In \cite{mehta2018}, the shape and size of the decision boundaries as well as size of the effective network (measured in terms of number of non-zero weights) are shown to provide further metrics of goodness of minima.

Another scenario that was proposed in \cite{coetzee1997488,dauphin2014identifying} and further 
confirmed in \cite{im2016empirical} is that the loss function of a deep network is typically proliferated 
by the large number of saddle points (and degenerate saddles \cite{sankar2017saddles}) 
compared to minima.  Gradient based optimization algorithms may get stuck at a saddle point 
rather than a minimum which slows down the learning. This is a typical nature of the types
of nonlinear multivariate cost functions one encounters in physics and chemistry \cite{doye2002saddle,2002JChPh.116.3777D,Wales04,Mehta:2011xs,mehta2014potential,Hughes:2014moa,Mehta:2013fza,mehta2015algebraic}. 
Several  ways to escape from saddle points provided no singular saddle points exist
have been developed \cite{dauphin2014identifying,pmlr-v40-Ge15,nesterov2006cubic}, and also in the presence of singular saddle points in certain specific cases \cite{anandkumar2016efficient,2017arXiv171007406L,panageas2016gradient}.

The ability of a minimization routine to find the global minimum also depends to the structure of the loss surface. 
In \cite{BallardDMMSSW17}, the loss surface of a single hidden layer feedforward neural network was shown
to have a \textit{single funnel} structure, i.e., asymmetric barrier heights between 
two adjacent local minima, ensureing that the routine quickly relaxes downhill 
towards the global minimum \cite{Wales03}. 

There are analytical and numerical results that either achieve the global minimum or 
construct necessary and sufficient conditions for a point to be the global minimum
for restricted classes of deep networks
\cite{NguyenLoss,haeffele2015global,arora2016understanding,yun2017global,feizi2017porcupine,zhang2017bpgrad}.

In \cite{baldi1989neural}, a detailed mathematical analysis of a class of a simpler model, the deep linear networks, was performed. Since then, the model has become one of the ideal test-grounds of ideas in artificial neural networks and deep learning \cite{baldi1995learning}. Below we first briefly describe the formulation of the model and then describe the previous results.

\subsection{Deep Linear Networks}

A deep linear network is an artificial neural network with multiple hidden layers with each neuron having a linear activation function. It is the linearity of the activation functions that separates deep linear networks from the deep nonlinear networks used in practice in which each neuron has nonlinear (or, at least, piecewise linear) activation function. The mean squared error for the deep linear networks with the usual $L_2$-regularization is defined to be \cite{kawaguchi2016deep,baldi1995learning}
\begin{equation}\label{equ:reg_loss}
    L(W) = \mathcal{\bar{L}}(W) + \lambda \sum_{i=1}^{H+1}\|W_i\|_2^2,
\end{equation}
with
\begin{equation}\label{equ:loss}
    \mathcal{\bar{L}}(W) = 
    \frac{1}{2} 
    \sum_{i=1}^m 
    \
    \bigg\| (W_{H+1} W_H \cdots W_1 X)_{\cdot,i} - Y_{\cdot,i} \bigg\|_2^2,
\end{equation}
where $\|.\|$ is the vector norm, $W_i$ is the weight matrix for the $i^{\rm th}$ layer 
with hidden layers from $i = 1,\dots,H$ and output layer $H+1$,
and $\lambda \geq 0$ is the regularization parameter. 
For $m$ data points in the training set, $d_x$ input dimensions, and $d_y$ output dimensions, 
the dimensions of $X$ and $Y$ are $d_x\times m$ and $d_y\times m$, respectively. 
Then, with $d_i$ hidden neurons in the $i^{\rm th}$ hidden layer, 
matrix multiplication yields that $W_1\in \mathbb{R}^{d_1\times d_x}, W_2\in \mathbb{R}^{d_2\times d_1}, \dots, W_{H+1}\times \mathbb{R}^{d_y \times d_H}$. 
We also denote $k=\mbox{min}(d_H, \dots, d_1)$, i.e., the number of neurons in the hidden layer with the 
smallest width. The number of weights, or variables, is $n = d_x d_1 d_2 \cdots d_H d_y$.

The simplicity of the deep linear network yields that it can approximate functions 
which are linear in $X$ and $Y$ though nonlinear in weights, 
whereas the real-world data may also possess nonlinearity in $X$. 
However, these networks contain most of the basic ingredients
of a typical deep nonlinear networks. Due to the network architecture, the loss function of the 
deep linear networks (Eqs.~\eqref{equ:reg_loss} and \eqref{equ:loss}) are still non-convex and 
non-trivial to analyze in a general setting. 
Understanding the loss surfaces of the deep linear networks also may enhance our understanding of the same for deep nonlinear networks.

\subsection{Earlier Works on Loss Surfaces of Deep Linear Networks}
Almost all the existing results for deep linear networks are for the $\lambda = 0$ case. For this case, deep linear network with $H=1$, 
under the assumptions that (1) $XX^T$ and $XY^T$ are invertible matrices, 
(2) $\Sigma = Y X^T (X X^T)^{-1} X Y^T$ has $d_y$ distinct eigenvalues, 
(3) $d_x = d_y$, i.e., an autoencoder, and 
(4) $k < d_x,d_y$, it was shown in \cite{baldi1989neural} that: 
\begin{enumerate}
    \item $\mathcal{\bar{L}}(W)$ is convex if either $W_1$ or $W_2$ are fixed, and the entries of the other vary. 
    \item Every local minimum is a global minimum.
\end{enumerate}
Moreover, \cite{baldi1989neural} also conjectured 
the following upon dropping the $H=1$ condition but retaining the other assumptions:
\begin{enumerate}
    \item $\mathcal{\bar{L}}(W)$ is convex if the entries of one $W_i$ vary while the others are fixed.
    \item Every local minimum is a global minimum.
\end{enumerate}
This conjecture was proven in more general settings of deep linear networks in \cite{kawaguchi2016deep,lu2017depth,2018arXiv180504938Z}, for deep linear complex-valued autoencoders with one hidden layer \cite{baldi2012complex}, as well as for deep linear residual networks \cite{hardt2016identity,yun2018critical}). Additionally, \cite{yun2017global} provides several necessary and sufficient 
conditions on global optimality based on rank conditions on the $W_i$ matrices for deep linear networks.

In \cite{zhou2017critical} (cf.~\cite{zhou2017characterization}), analytical forms of the 
stationary points (including minima) characterizing the values of the loss function were 
presented for deep linear as well as for certain limited cases of unregularized deep nonlinear networks.
The aforementioned necessary and sufficient conditions for global optimality were also 
reformulated with the help of the analytical form of the critical points.

Layer-wise training of deep linear networks was investigated from the dynamical systems point of view 
in \cite{saxe2013exact} (see also \cite{2018arXiv180600730B}) and was concluded that the learning speed can remain finite even 
in the $H\to \infty$ limit for a special class of initial conditions on the weights,
likely due to no local minima present in the landscape.

The $\lambda > 0$ case is considered in \cite{taghvaei2017regularization}
by modeling a linear networks (though, not a deep linear network) with $L_2$-regularization term as a
continuous time optimal control problem.  The problem of characterizing the critical points of the
deep linear networks was reduced to solving a finite-dimensional nonlinear matrix-valued equation. 
Here, the continuous-time is essentially a surrogate index for layers and the final weight matrix 
was assumed to be square. It was shown that for a special case of the model, even for small amount of regularization, saddle points emerge.

\subsection{Our Contribution}
The main conceptual contribution in this paper is to identify solving the gradients of deep networks  as a 
computational algebraic geometry \cite{CLO:07,CLO:98} problem. We review the existing literature related to the optimization landscape and put our algebraic geometry point of view into perspective. The other key contributions from this viewpoint are summarized as follows:
\begin{enumerate}
    \item We clarify on various definitions of \textit{flat} minima, and distinguish the geometric definition of flat minima from the other definitions. We then show their existence in the unregularized landscapes of deep linear networks.
    \item Then, we prove that a straightforward extension of $L_2$ regularization can guarantee to remove all these flat minima: these flat minima are only an artifact of the underlying residual symmetries of the equations and can be removed using, for example, the generalized  $L_2$-regularization.
    \item We take up a novel question on deep learning loss surfaces: how many isolated stationary points and, more specifically, minima are there in a typical deep learning loss surface? With the help of algebraic geometry, we provide the first results in this direction on upper bounds on the number of stationary points for deep linear networks. Obviously, these upper bounds provide strict upper bounds on the number of local minima.
    \item Next, we custom-make a numerical algebraic geometry method which guarantees to find \text{all} stationary points of the deep linear networks for modest size systems. With all stationary points at hand, we obtain further novel insights on the loss landscapes of the deep linear networks, in addition to explicitly showing that the model exhibits local minima which are \textit{not} global minima for the regularized case.
    
\end{enumerate}

The remainder of the paper is organized as follows: in Section \ref{sec:alg}, a brief introduction of algebraic geometry is provided, and a relation between algebraic geometry and deep linear networks is discussed. We also put our approach in perspective with respect to other attempts to apply algebraic geometry methods to machine learning. In Section \ref{sec:flat_saddle}, we show how the flat stationary points of unregularized gradient equations can be removed using a generalized regularization. In Section \ref{sec:Bound}, we provide algebraic geometry based upper bounds on number of stationary points of the gradient equations. In Section \ref{sec:num_results}, we introduce and apply the polynomial homotopy continuation method and provide results for modest size systems. In Section \ref{sec:conclusion}, we discuss our findings in more details and conclude.

\section{Algebraic Geometric Interpretation of Deep Linear Networks}\label{sec:alg}
In this section, we argue that solving the gradient equations of deep linear networks can be viewed as an algebraic geometry problem, and briefly introduced algebraic geometry terminologies while distinguishing our algebraic geometry interpretation of the problem with previous attempts.

In this paper, we identify solving the gradient equations of deep learning as an algebraic geometry\cite{CLO:07,CLO:98} problem. There have been various attempts to relate algebraic geometry and machine learning, e.g., an abstract relation between statistical learning methods and algebraic geometry has been extensively investigated \cite{watanabe2009algebraic}. Machine learning methods have been used to improve computational algebraic geometry methods such as in computing cylindrical algebraic decomposition \cite{huang2018using} and to find roots of certain polynomials \cite{huang2004neural,perantonis1998constrained,mourrain2006determining}. Neural networks have also been shown to effectively learn data whose target function is a polynomial \cite{andoni2014learning} (see also \cite{8027142,knoll2017loopy}. In the present paper, on the other hand, we explore the loss landscape by interpreting solving the gradient systems of deep learning as an algebraic geometry problem. The algebraic geometry interpretation allows us to investigate the gradient equations for both the regularized and unregularized cases, as well as for arbitrary data and size of all the matrices, equally well. Though we focus on deep linear networks in this paper, the deep learning problem can also
be cast as an algebraic geometry problem in the presence of all the conventional activation functions. 

\subsection{The Gradient Equations are Algebraic Equations}
The critical points of the objective function $L$ are points at which
all partial derivatives are equal to zero, i.e., satisfy
the gradient equations $\nabla L = \mathbf{0}$.
This gradient equations form a system of equations which
is nonlinear in its variables, i.e., the entries of $W_i$.
This system is naturally an algebraic system since
each equation is polynomial in the variables.  

Let $W = W_{H+1}\dots W_1$, 
$U_i^\top = \prod_{j=i+1}^{H+1} W_j^\top$, 
and $V_i^\top = \prod_{j=1}^{i-1} W_j^\top$.
Then, $\dfrac{\partial L}{\partial W_i}$ is a matrix whose~$(j,k)$ entry is the 
partial derivative of $L$ with respect to the $(j,k)$ entry of $W_i$.
Hence, 
\begin{equation}\label{equ:grad}
    \frac{\partial L}{\partial W_i} = 
    U_i^\top 
    \left(
        W
        \left( \sum_{k=1}^m \mathbf{x}_k \mathbf{x}_k^\top \right) - 
        \left( \sum_{k=1}^m \mathbf{y}_k \mathbf{x}_k^\top \right)
    \right) 
    V_i^\top + \frac{\lambda}{2} W_i.
\end{equation}

Thus, each partial derivative is polynomial in the entries of $W_1,\dots,W_{H+1}$.
Therefore, studying the critical points of $L$ is equivalent to studying the
solution set to a system of polynomial equations, i.e., the gradient equations,
which is the central question in the field of \emph{algebraic geometry}.

\subsection{A Brief Introduction to Algebraic Geometry}

In the context of deep linear networks, critical points are {\em real} solutions
to the gradient equations.  It is common in algebraic geometry \cite{CLO:07,CLO:98} to simplify the
problem by computing all solutions over the complex numbers 
since the complex numbers form an algebraically closed field, i.e.,
every univariate polynomial equation with complex coefficients has at least one complex solution.

An {\em algebraic set} is the solution set
of a collection of polynomial equations.  That is, the
algebraic set associated to the 
polynomial system $f(x) = (f_1(x),\dots,f_m(x))$,
where $x=(x_1,\dots,x_n)$ in $\bC^n$~is
$$V(f) = \{x\in\bC^n~|~f_i(x) = 0, i = 1,\dots,m\}.$$
The real points in $V(f)$ is simply $V_\bR(f) = V(f)\cap\bR^n$.
An algebraic set $A$ is {\em reducible} if there exists
algebraic sets $\emptyset\subsetneq B_1,B_2 \subsetneq A$
such that $A = B_1\cup B_2$, otherwise, $A$ is said
to be {\em irreducible}. Every algebraic set can be 
presented uniquely, up to reordering, as a finite 
union of irreducible algebraic sets yielding its 
{\em irreducible decomposition}.

Each irreducible algebraic set $A$ has a well-defined 
{\em dimension}.  Every dimension $0$ irreducible algebraic set $A$
is a singleton, i.e., of the form $A = \{p\}$, in which case $p$ is called
an {\em isolated solution} to the corresponding
polynomial system.  A positive-dimensional
algebraic set consists of infinitely many points, e.g., a curve
has dimension $1$ and a surface has dimension $2$.  
In the context of the gradient equations, isolated solutions 
correspond with {\em isolated stationary points} and positive-dimensional algebraic sets consist of {\em flat~stationary~points}.

\subsection{Difference Between Complexifying the Gradient Equations and Complex Loss Functions}
We note that neural networks with complex-valued weights (and complex-valued inputs and outputs) have been studied in the past \cite{georgiou1992complex,zemel1995lending,kim2003approximation,nitta1997extension,akira2003complex} and have gained renewed interest in deep learning \cite{reichert2013neuronal,guberman2016complex,danihelka2016associative,wisdom2016full,trabelsi2017deep} for its use in simultaneously modelling phase and amplitude data. In particular, back-propagation for complex-valued neural networks was developed in \cite{nitta1997extension}. In \cite{nitta2003solving}, it was shown that the XOR data which cannot be solved with a single real-valued neuron in the hidden layer, can be solved with a complex-valued network. Such complex-valued neural networks were shown to have better generalization characteristics \cite{hirose2012generalization}, and faster learning \cite{arjovsky2016unitary}, in addition to biological motivations \cite{reichert2013neuronal}.

We note that in the aforementioned formulation of the deep complex networks consider complex weights, inputs and outputs and hence the corresponding loss function is also complex-valued. On the other hand, in the present paper, we start from the conventional real-valued weights, inputs and outputs, with the loss function also being real-valued. Then, we merely complexify the gradient equations in that we assume weights living in the complex space and inputs and outputs living in the real space. In other words, the former is fundamentally a complex-valued set up whereas in the latter case the weights are complexified for the computational analysis purpose.

\section{Flat Stationary Points and Regularization}\label{sec:flat_saddle}
In this section, we briefly review the existing literature on flat minima in deep learning.
In \cite{hochreiter1997flat,hochreiter1995simplifying}, an algorithm to search for some (but not provably all) \textit{acceptable} (i.e., almost) flat minima, which are large connected regions of minima at which the training error was below a threshold, was proposed. Such acceptable flat minima correspond to weights many of which may be specified with low precision (hence, with fewer bits of information). In these references it was also argued that these minima also correspond to low complexity networks.

In \cite{keskar2016large}, it was empirically shown that SGD based methods tend to converge to sharp (flat) minima with large (small) batch sizes. In \cite{jastrzkebski2018finding,jastrzkebski2017three}, it was argued that higher (lower) ratio between learning rate and batch size pushes the SGD towards flatter (sharper) minima, and that the flatter minima generalized better than sharper minima. In \cite{chaudhari2016entropy}, an entropy-SGD was proposed that actively bias the optimization towards flat minima of specific widths (cf. \cite{baldassi2016local,baldassi2016learning,baldassi2016unreasonable,zhang2018energy}). However, later on, in \cite{dinh2017sharp}, the above definitions of flatness of minima were formalized and it was then argued that deep networks do not necessarily generalize better when they converge to ``flat`` minima (as defined above) than sharp minima because one can reparametrize the loss function that correspond to equivalent models but possessing arbitrarily sharp minima. 

In the current paper, we are interested in the exactly flat saddles and minima, i.e., the connected components of the stationary points on which the loss function is precisely constant, whose existence is well-known since the works of \cite{brockett1976some,chen1993geometry,saad1995line,kuurkova1994functionally} (see \cite{amari2006singularities} for a review). Such degenerate regions, sometimes referred to as \textit{neuromanifolds}, are quite common \cite{watanabe2007almost} not only in artificial neural networks loss landscapes due to various symmetries \cite{sussmann1992uniqueness,chen1993geometry,NIPS2015_5797,NIPS2015_5797} of the corresponding loss functions. At such solutions, the Fischer information matrix tends to be singular and the traditional gradient descent algorithms are known to slow down.

To be sure, the hessian matrix of the loss function can be singular at either isolated singular solutions (i.e., multiple roots) as well as at a non-isolated degenerate solution region. In \cite{sankar2017saddles}, it was shown using numerical experiments for modest size deep neural networks that the available SGD based optimization routine converged to degenerate saddle points at which the Hessian matrix not only has many positive and negative eigenvalues but also multiple zero eigenvalues. Moreover, they showed that the number of zero eigenvalues increases with increasing depth. It was argued that for a good training it is enough that deep neural network models converge at degenerate saddle points as long as the training error is low. Whereas, in \cite{sagun2016eigenvalues}, by computing the eigenvalues of the Hessian of deep nonlinear networks after training as well as at random points in the configuration space, it was shown that a vast number of eigenvalues were zero. Hence, most of the directions in the weight space of these networks are flat moving in which leads to no change in the loss function. In \cite{sagun2017empirical}, it was shown that though small and large batch gradient descent appeared to converge to seemingly different minima, a straight line interpolation between the two did not contain any barrier, implying that the two regions may be in the same basin of attraction. In the present paper, we make a distinction between isolated singular solutions and flat minima. We also carry forward the distinction made in \cite{sagun2017empirical} between almost flat minima within which the loss function is almost constant, and the flat minima within which the loss function is precisely constant. The former should be referred to as the ``wide`` minima.

In terms of algebraic geometry, a stationary point is flat if it is not an isolated solution of the gradient equations.  
Hence, each flat stationary point lies on a positive-dimensional
component. For the purpose of this paper, we focus on complex positive-dimensional stationary points which may include real positive-dimensional solutions because in the next section where we device a method to remove positive-dimensional stationary solutions, we remove all complex and real positive-dimensional stationary points.

We present a few examples to show some explicit results.
The first example arise in \cite{panageas2016gradient}.

\noindent \textit{Example-1:}
The gradient of $f(x,y,z) = 2 x y+ 2 x z - 2 x -y - z$ 
is $\nabla(f) = \{2 y + 2 z - 2, 2 x -1, 2 x -1\}$.
The set of stationary points, which satisfies $\nabla(f) = 0$,
is the line defined by $x - 1/2 = y+z-1 = 0$, i.e., in the complex space, the solution has dimension 1.
At every point on this line, one has $f(x,y,z)=-1$.

\noindent \textit{Example-2:}
For $H=1$, $m=5$, $d_x = d_y = 2$, and $d_1 = 1$ with $\lambda = 0$, we consider the data matrices
\[ X = 
\begin{bmatrix}
    7  & -8 & 3 & -5 & 10 \\
    -7  & 10 & 6 & -2 & 6 
\end{bmatrix}, \]
\[
Y=
\begin{bmatrix}
    9 & 9 & -8 & 1 & 10 \\
    10 & 3 & -8 & 9  & 10
\end{bmatrix}.
\]
The stationary solutions of Eqs. \eqref{equ:grad} consists of three irreducible components:
a point is an isolated saddle,
a curve consisting of flat saddle points,
and a curve consisting of flat minima described as follows.
The point is $W_1 = 0 \in \bR^{1\times2}$ and $W_2 = 0\in\bR^{2\times1}$.
The flat saddle points and flat minima have the form
$$W_1 = \alpha\cdot\widehat{W_1} \hbox{~~and~~} W_2 = \alpha^{-1}\cdot\widehat{W_2}$$
for any $\alpha\neq 0$.  For example, the flat saddle points approximately have
$$\widehat{W_1} = \left[\begin{array}{cc}
1 & 9.6330
\end{array}\right]
\hbox{~~and~~}
\widehat{W_2} = \left[\begin{array}{c}
0.0206 \\ -0.0180
\end{array}\right]
$$
while the flat minima approximately have
$$\widehat{W_1} = \left[\begin{array}{cc}
1 & 0.0696
\end{array}\right]
\hbox{~~and~~}
\widehat{W_2} = \left[\begin{array}{c}
0.2664 \\ 0.3045
\end{array}\right].
$$

\begin{remark}
This example generalizes to all critical points in the unregularized case, i.e., $\lambda = 0$.
That is, if $(W_1,\dots,W_{H+1})$ is a critical point, then
so is $(A_1W_1,A_2W_2A_1^{-1},\dots,W_{H+1}A_H^{-1})$.
Hence, if there is a critical point with some $W_i\neq 0$, 
then there are always flat critical points in the unregularized case.
\end{remark}

The traditional $L_2$-regularization with single $\lambda$ (Eq.~\ref{equ:reg_loss}) is also not necessarily enough to remove the the flat stationary points. The following provides
a simple illustrative example showing that this need
not be the case.  

\noindent \textit{Example-3:}
For $H = 1$, $m = 3$, and $d_x = d_y = d_1 = 2$, we consider
$$X = \left[\begin{array}{ccc} 1 & 2 & 3\\ 1 & 2 & 3 \end{array}\right]
\hbox{~~and~~}
X = \left[\begin{array}{ccc} 1 & 2 & 3\\ 1 & -3 & 2 \end{array}\right]
.$$ 
For any $\lambda\geq0$ and $a\geq 0$, the following is a family
of flat critical points:
$$
\hbox{\small $
W_1 = \left[\begin{array}{cc}
a & a \\ \gamma(a,\lambda) & \gamma(a,\lambda)
\end{array}\right],
W_2 = \sqrt{\frac{2}{197}}\left[\begin{array}{cc}
14a & 14 \gamma(a,\lambda) \\ a & \gamma(a,\lambda)
\end{array}\right]$}$$
where
$$\gamma(a,\lambda) = \sqrt{\sqrt{394}/56 - a^2 - \lambda/28}$$
For example, if $0 < \lambda < \sqrt{\frac{197}{2}}$, then this component consists of flat minima
which are real for $0 \leq a \leq \sqrt{7(\sqrt{197/2} - \lambda)}/14$.

\subsection{Regularization of flat critical sets}
We begin the discussion of removing flat minima from the landscapes of loss functions by pointing out two observations: First, in \cite{BallardDMMSSW17,mehta2018}, where the goal of the study was to numerically investigate the loss landscape of a deep nonlinear neural network with one hidden layer and $\tanh$ activation functions, it was noted that the constant zero eigenvalues disappeared as soon as the $L_2$-regularization term was non-zero. Here, all the weights including the bias weights were regularized \cite{mehta2018}. However, this observation may not directly apply in general because the continuous symmetries present in more complex systems may depend on the network architectures, activation functions, data, etc. Second, in \cite{Nerattini:2012pi}, a spin glass model called the XY model was found to exhibit residual continuous symmetries, and a generalized regularization term was used to remove the continuous symmetries. 

As outlined above, the existence of flat or degenerate critical set 
is a very common phenomenon in the study of deep linear and nonlinear networks in general.
At any point in a flat critical set of $\lfunc$, the Hessian matrix of $\lfunc$
has at least one zero eigenvalue.
Such a zero eigenvalue of the Hessian matrix signifies certain degree of freedom
in the weight matrices.
That is, there are directions in which weights infinitesimally change
without violating the gradient equations.

From a computational point of view, flat critical sets introduces many unnecessary
difficulties: For example, a simple solver based on Newton's iterations will encounter
numerical instabilities near a flat critical set.
From a purely theoretical point of view, flat critical sets indicate the
training data set and the network structure are not sufficient to determine
the optimal configuration of the weights.
In this section, we outline a ``regularization'' technique that could perturb the 
loss function $\mathcal{\bar{L}}(W)$ ever so slightly so that 
all the critical points become non-degenerate (isolated) critical points.
That is, such perturbation would remove the flatness from all critical points.

Recall that for a smooth function $f : \R^n \to \R^n$,
$\mathbf{v} \in \R^n$ is said to be a \emph{regular value}
if for each $\mathbf{x} \in \R^n$ such that $f(\mathbf{x}) = \mathbf{v}$,
the Jacobian matrix $Df$ is nonsingular at $\mathbf{x}$.
\textbf{Sard's Theorem} ~\cite{abraham_transversal_1967} states that almost all $\mathbf{v} \in \R^n$
are regular values (in the sense of Lebesgue measure).
This result can be generalized into a stronger result on parametric systems
that fits our current discussions:
Let $f(\mathbf{a},\mathbf{x}) : \R^m \times \R^n \to \R^n$ be a smooth function.
\textbf{Generalized Sard's Theorem}~\cite{abraham_transversal_1967} 
states that if $\mathbf{0}$ is a regular value for $f$, 
then for almost all $\mathbf{a} \in \R^m $, $\mathbf{0}$ is a regular value of the function 
$f_{\mathbf{a}}(\mathbf{x}) = f(\mathbf{a},\mathbf{x})$ with the parameter $\mathbf{a}$ fixed.
In the following, we adapt this idea to the context of deep linear networks.

Motivated by the aforementioned observations and the Generalized Sard's Theorem, 
we device a regularization for the deep linear networks: Given $H+1$ matrices with positive real entries
$\Lambda = (\Lambda_1,\dots,\Lambda_{H+1})$ with each $\Lambda_i$ having the same size as $W_i$,
we can consider a generalized Tikhonov regularization of $\mathcal{\bar{L}}$ given by
\begin{equation*}
    \mathcal{L}^\Lambda =
    \mathcal{\bar{L}}(W) + 
    \frac{1}{2} (
        \| \Lambda_1 \circ W_1 \|_F^2 +
        \cdots 
       + \| \Lambda_{H+1} \circ W_{H+1} \|_F^2
    ),
\end{equation*}
where $\Lambda_i \circ W_i$ denotes the Hadamard product (entrywise product)
between  $\Lambda_i$ and $W_i$.
That is, each term in $\Lambda_i \circ W_i$ is of the form of
$\lambda_{i,j,k} w_{i,j,k}$,
where $\lambda_{i,j,k}$'s, the entries of $\Lambda_i$'s, 
are small positive real numbers that serve as  penalty coefficients.
Therefore the minimization problem for $\mathcal{L}^\Lambda$ attempts to
minimize $\lfunc$ and at the same time minimize each entries of the weight matrices.
Note here that the penalty on each entry of the weight matrices could
potentially be different.
It is straightforward to verify that
\begin{equation}\label{equ:reg-grad}
    \frac{\partial \lfunc^\Lambda}{\partial W_i} =
    U_i^\top (W X X^\top - Y X^\top) V_i^\top + \Lambda_i \circ W_i.
\end{equation}

When entries of $\Lambda_i$s are small positive real numbers,
we can see the above gradient system is a slightly perturbed version of the
original gradient system $\nabla \lfunc$.
In the following, we demonstrate that this construction is sufficient 
to turn flat critical set of $\lfunc$ into isolated nondegenerate
critical points.
That is, the flatness of the critical points will be removed.

First, we shall show the above regularization technique is sufficient to
``desingularize'' all \emph{dense critical points}.
Here, a dense critical point of $\lfunc^\Lambda$ is a (real) solution to
$\frac{\partial \lfunc^\Lambda}{\partial W_i} = 0$ for each $i$ 
for which $W_i$ contains no zero entries,
i.e., all weight matrices are dense matrices.

\begin{theorem}[Regularity of dense critical points]\label{thm:reg-dense}
    For almost all choices of $\Lambda$,
    all dense (real) critical points of $\lfunc^\Lambda$
    are isolated and nondegenerate.
\end{theorem}
\noindent\begin{proof}
    Let $W = (W_1,\dots,W_{H+1})$ collect all the weight matrices,
    and let $m$ be the total number of entries in all these matrices.
    Consider the open set $(\mathbb{R}^*)^m = (\mathbb{R} \setminus \{0\})^m$.
    Let $F(W_1,\dots,W_{H+1},\Lambda) = (\frac{\partial \lfunc^\Lambda}{\partial W_i})_{i=1}^{H+1}$
    be the gradient of $\lfunc^\Lambda$ with respect to $W_1,\dots,W_{H+1}$.
    Here, we include the parameters --- coefficients in $\Lambda$ as variables.
    
    \[
        \frac{\partial F}{\partial \lambda_{i,j,k}} = w_{i,j,k}
    \]
    
    The Jacobian matrix of $F$ is a $m \times 2m$ matrix.
    Since $\frac{\partial F}{\partial \Lambda}$ is a diagonal matrix whose diagonal entries
    are $w_{i,j,k} \ne 0$,
    we can conclude that the Jacobian matrix is of rank $m$,
    i.e., full row rank.
    Therefore $\mathbf{0}$ is a regular value for the map
    $F : (\R^*)^m \times \R^m \to \R^m$.
    Then by the Generalized Sard's Theorem~\cite{abraham_transversal_1967}, 
    for almost all choices of $\Lambda \in \R^m$,
    $\mathbf{0}$ is a regular value for the map $F_\Lambda : (\R^*)^m \to \R^m$ given by
    $F_\Lambda = F(\cdot,\Lambda)$.
    Consequently, for any $W$ such that $F_{\Lambda}(W) = F(W,\Lambda) = \mathbf{0}$,
    the square Jacobian matrix $\frac{\partial F_\Lambda}{\partial W}$ must be of full column rank,
    i.e., nonsingular, which implies $W$ must be a nonsingular solution of the equation.
    By the Inverse Function Theorem, such a solution must also be geometrically isolated.
\end{proof}

Here, a critical point is considered \emph{isolated} (a.k.a. geometrically isolated) 
if there is a neighborhood in which it is the only critical point.
An isolated critical point is considered \emph{nondegenerate} when the Hessian matrix
at this point is nonsingular.
The ``almost all choices'' in the above statement is to be interpreted in the sense of 
Lebesgue measure. It is also sufficient to take a probabilistic interpretation:
if the entries of $\Lambda$ are chosen at random, then with probability one,
the above theorem holds.

\begin{remark}
Instead of randomly drawing each $\lambda_{ijk}$ separately, one can also consider $\lambda_{ijk} = \lambda + \rho_{ijk}$. Then, the $\rho_{ijk}$s are drawn from a random distribution once for all, and adjusting the regularization again becomes only one parameter problem.
\end{remark}

The above regularization result can also be generalized to ``sparse'' cases
which are desired in actual application.
For instance, in convolutional neural networks, the first layer is
generally highly structured and very sparse as it represents
the application of convolution matrices.
Similarly, many real world applications have specific sparsity pattern in mind.
We therefore generalize the above result with respect to certain sparsity pattern.
A sparsity pattern for the weight matrices is a set $\mathcal{N}$
of indices of the form $(i,j,k)$ specifying the nonzero positions.
We say the matrices $(W_1,\dots,W_{H+1})$ have the sparsity pattern $\mathcal{N}$
if for each $(i,j,k) \in \mathcal{N}$, the $(j,k)$ entry of $W_i$ is nonzero
while all other entries are zero.
We can generalize the above theorem to weight matrices having a given sparsity pattern,
and the dense cases of Theorem~\ref{thm:reg-dense} will be the special case that
require all entries to be nonzero.

\begin{theorem}[Regularity of sparse solutions]
    Given a sparsity pattern $\mathcal{N}$,
    for almost all choices of $\Lambda$,
    all (real) solutions of the gradient system $\nabla \lfunc^\Lambda = \mathbf{0}$
    having the sparsity pattern $\mathcal{N}$ are geometrically isolated and nonsingular.
\end{theorem}
\noindent\begin{proof}
    Let $W^{\mathcal{N}}$ be the set of all $w_{i,j,k}$'s for which $(i,j,k) \in \mathcal{N}$.
    That is, $W^{\mathcal{N}}$ collect all the nonzero entries in the weight matrices.
    By fixing all the remaining entries to zero, the gradient equations
    $\frac{\partial \lfunc^\Lambda}{\partial W_i}$ for $i=1,\dots,H+1$ under the
    regularization can be considered as a system in $W^{\mathcal{N}}$ only.
    
    Following the previous proof, we can define $m = | W^{\mathcal{N}} |$ and
    $F (W^{\mathcal{N}}, \Lambda^{\mathcal{N}})$ to be the system of gradient equations
    with $\Lambda^{\mathcal{N}}$ (entries in $\Lambda$ corresponding to $W^{\mathcal{N}}$)
    also considered to be variables.
    Then as in the previous case, the Jacobian matrix of $F$ is a $m \times 2m$ matrix
    with $\partial F / \partial \Lambda^{\mathcal{N}}$ being a diagonal matrix with
    nonzero diagonal entries $w_{i,j,k}$ for $(i,j,k) \in \mathcal{N}$.
    Consequently, this Jacobian matrix also has full row rank.
    By the generalized Sard's theorem, we can conclude that for almost all choices
    of $\Lambda^{\mathcal{N}} \in \R^m$, all solutions to
    $F_{\Lambda^{\mathcal{N}}}(W^{\mathcal{N}}) = F(W^{\mathcal{N}},\Lambda^{\mathcal{N}}) = \mathbf{0}$
    must be geometrically isolated and nonsingular.
\end{proof}
Note that the regularization $\lfunc^\Lambda$ is constructed as a perturbation
of the original loss function $\lfunc$ with small penalty terms added to
also minimize the magnitude of each weight coefficient.
The theory of \emph{homotopy continuation method} \cite{79:allgower} also guarantees that for
sufficiently small perturbation, this process can be reversed.
The following is an immediate consequence of the Implicit Function Theorem.

\begin{proposition}
    For sufficiently small regularization coefficients $\Lambda$,
    as all entries of $\Lambda$ shrink to 0 uniformly,
    the critical points of $\lfunc^\Lambda$ also move smoothly
    and either converge to regular critical sets of $\lfunc$
    or diverge to infinity.
\end{proposition}

Here, ``diverge to infinity'' means as the perturbation coefficients in $\Lambda$ shrink zero,
certain coordinates in some of the critical point of $\lfunc^\Lambda$ grow unboundedly.

\begin{remark}
More rigorous description of this phenomenon of diverging solutions can be given in terms of 
\emph{projective space} \cite{CLO:98,CLO:07} which encapsulate infinity as an actual place in the space.
In that sense, certain critical points of $\lfunc^\Lambda$ may converge to
``saddle points at infinity'' of $\lfunc$.
\end{remark}

Another important observation from the homotopy point of view is that 
while this perturbation slightly alters the loss landscape, 
any global minimum will survive in the following sense.
The following is an immediate consequence of \cite{ConnectedComponents}
as well as the Implicit Function Theorem.

\begin{proposition}
    For sufficiently small regularization coefficients $\Lambda$,
    as all entries of $\Lambda$ shrink to 0 uniformly,
    there is at least one critical point of $\lfunc^\Lambda$ that 
    will converge to a global minimum of~$\lfunc$.
\end{proposition}

Below we show the regularization technique implemented for example 1.

\noindent \textit{Example-4:}
The gradient of $f(x,y,z) = 2 x y+ 2 x z - 2 x -y - z + (\frac{2}{1000} x^2 + \frac{1}{1000} y^2 + \frac{3}{1000} z^2)$ 
is $\nabla(f) = \{\frac{x}{250} + 2 y + 2 z - 2, 2 x + \frac{y}{500} - 1, 2 x + \frac{3}{500} z - 1\}$.
The set of stationary points, which satisfies $\nabla(f) = 0$,
is $x\sim 0.49925, y\sim 0.74925, z \sim 0.24975$, and the dimension of the solution is $0$.

\begin{remark}
There have been various methods proposed which escapes flat saddle points (in the wide minima sense) in the absence of singular saddles \cite{dauphin2014identifying,pmlr-v40-Ge15,nesterov2006cubic}. Recent attempts have also been made to extend such methods in the presence of singular saddle points in limited cases \cite{anandkumar2016efficient,2017arXiv171007406L,panageas2016gradient}. With the generalized $L_2$-regularization, only the former set of methods may be required to escape saddles and achieve a deeper minimum.
\end{remark}

\section{Estimates on the Number of Isolated Solutions}\label{sec:Bound}

Now that after the generalized $L_2$ regularization, we are left with only isolated stationary points, in this section, we focus on estimates on the number of \textit{isolated} solutions of Eqs. ~\eqref{equ:reg-grad}.

\subsection{Upper Bounds on the Number of Isolated Stationary Points}

Algebraic geometry interpretation of the gradient systems of the deep linear networks allows us to utilize different theorems on root-counts of the number of complex solutions to estimate number of stationary points of the system. 
To that end, suppose that $f(x) = (f_1(x),\dots,f_n(x))$
is a polynomial system where $x\in\bC^n$, i.e., $f$ is a {\em square}
system of polynomials.  

\subsubsection{Classical B\'ezout Bound:} The simplest upper bound on 
the number of isolated complex stationary points in $V(f)$ 
is called the {\em classical B\'ezout bound} (CBB) which is simply 
the product of the degrees of the polynomials in $f$, namely
$\prod_{i=1}^n \deg f_i$.  In fact, this bound, and all
others discussed below, are generically sharp
with respect to the structure that they capture. 

From (\ref{equ:grad}) and the definition of $\lfunc^\Lambda$,
we can see that the leading terms in each polynomial are formed by the product of
$2H+1$ matrices, therefore each polynomial is of degree $2H+1$.
The CBB is therefore the product of these degrees:

\begin{proposition}
    The regularized loss function $\lfunc^\Lambda$ has at most
    $(2H+1)^n$ complex isolated critical points
    where $n$ is the total number of weights.
\end{proposition}

The CBB is sharp when the system of equations is completely dense, i.e., each polynomial contains all the possible monomials with degrees equal or less than the degree of the polynomial. The systems arising from real-world applications, same as Eqs.~\eqref{equ:reg-grad}, are however sparse. Sparse systems have in general significantly smaller number of complex isolated solutions than the CBB. Hence, we need a root-count which takes the sparsity structure of the polynomial systems such as the Bernshtein-Kushnirenko-Khovanskii bound.

\subsubsection{BKK Bound:} Another refinement of the B\'ezout bound takes into consideration
the geometry of the ``monomial structure'' of the polynomial system $f$.

This number is known as the \emph{Bernshtein-Kushnirenko-Khovanskii Bound},
or simply \emph{BKK Bound}~\cite{bernshtein_number_1975,huber_polyhedral_1995}, 
and it is given by a geometric invariant defined on the monomial structure ---
the \emph{mixed volume} of the convex bodies created by the set of monomials
appear in $f$ (i.e., the \emph{Newton polytopes} of $f$).
This bound is given for number of solutions in $\mathbb{C}^n - \{\textbf{0}\}$ but can also be extended to a bound on number of isolated solutions in
$\mathbb{C}^n$~\cite{li_bkk_1996,rojas_counting_1996}.

The BKK bound provides a much tighter refinement on this upper bound.
As shown in Table~\ref{tab:bkk1}, the BKK bound the gradient system of
$\lfunc^\Lambda$ is much lower than the B\'ezout number.

\subsection{Analytical Results for Mean Number of Real Solutions of Random Polynomial Systems}
There are only a handful of results known for the upper bounds on the number of isolated real stationary points of polynomial loss functions \cite{dedieu2008number}, or for upper bounds on the number of isolated real solutions of systems of polynomial equations \cite{kac1943average,kac1948average,edelman1995many,blum1998f,kostlan2002expected,azais2005roots,armentano2009random}. 

To gain further insight on the number of real stationary points of Eq.\eqref{equ:reg_loss} (with entries of $X$ and $Y$ picked randomly from a random distribution), we compare the existing analytical results for the mean number of real stationary points of random polynomial cost functions. The most general random polynomial cost function is written as:
\begin{equation}\label{eq:random_poly}
    F(\textbf{x}) = \sum_{|\alpha| < (2H+2)} a_{\alpha} x_1^{\alpha_1}\dots x_{n}^{\alpha_n},
\end{equation}
with $n$ being the number of variables and $(2H+2)$ is the highest degree of the monomials. $\alpha = (\alpha_1, \dots, \alpha_n) \in \mathbf{N}^n$ is a multi-integer with $|\alpha| = \alpha_1 + \dots + \alpha_n$. Here, $a_\alpha$ are random coefficients i.i.d. drawn from the Gaussian distribution with mean 0 and variance 1. In 
\cite{dedieu2008number}, it was shown that the mean number of real stationary points of this cost function, i.e., the mean number of real solutions of the corresponding gradient system $\frac{\partial F(\textbf{x})}{\partial x_i} = 0$ for $i=1, \dots, n$, is:
\begin{equation}\label{eq:DM_bound_for_deep_linear_network}
    \mathcal{N}_{DM} (H,n) = \sqrt{2}(2H+1)^{\frac{n+1}{2}},
\end{equation}
i.e., the mean number of real stationary points of the random polynomial of the same degree and number of variables as the loss function of the deep linear network. This result also yields that the mean number of real stationary points of such a dense random polynomial function is significantly smaller than the corresponding CBB.

\subsection{A Few Words on the Equations for the Zero Training Error}
In this subsection, we briefly consider the problem of finding a special type of minima, called the zero training error minima, as these were recently studied for certain class of deep learning models in \cite{liao2017theory}. In \cite{liao2017theory}, deep nonlinear networks with rectified linear units (ReLUs) were considered and the ReLUs were approximated with polynomials of certain degree. Then, the classical B\'ezout theorem was applied to conclude that for the case when there are more weights than the number of data points, there always are infinitely many global minima expected. 

For deep linear networks, the zero training error minima are the minima which satisfy the equations $\mathcal{\bar{L}}(W) = 0$, i.e.,
\begin{equation}\label{eq:zero_error}
(W_{H+1} W_H \cdots W_1 X)_{\cdot,i} - Y_{\cdot,i} = 0,
\end{equation}
for all $i=1,\dots, m$. It must be emphasized that these minima may only exist if the model can fit \textit{all} the training data perfectly well. Except for some special cases, it is also difficult to know if such minima exist for a chosen model a priori. Clearly, such minima are the global minima of the model for the specific dataset. Here, we \textit{assume} that such zero training error minima do exist for our deep linear networks for the given data matrices. Then, Eqs. (\ref{eq:zero_error}) is again a system of $m$ polynomial equations in $n$ variables. In short, for $H \ge 1$ and $Y \ne 0$, 
the zero training error minima system (Eqs. \eqref{eq:zero_error}) has no isolated solutions. For the case when $m = n$, the CBB is $(H+1)^n$ complex isolated solutions.

We emphasis that the assumption that such zero training error minima do exist is a very strong one as it means that each data point is exactly fit, which either may not occur in practice or may be a case of over-fitting.

\begin{remark}
For the underdetermined systems, the CBB and BKK are actually bounds on the number of connected components (flat stationary points).
The existence of positive-dimensional components reduces
the maximum number of isolated solutions. In fact, even for an apparently underdetermined system, it may be possible to have only isolated stationary points. However, except for the special case of $m=n$, these bounds do not provide any detailed information about number of flat stationary points.
\end{remark}

\subsection{Symmetrical Solutions}
In this subsection, we show the existence of some symmetries in the solutions of the gradient equations of the deep linear networks.
\begin{proposition}
    When $H=1$, if $W_0^*$ and $W_1^*$ form a solution to system~\eqref{equ:reg-grad}, then the vector formed by simultaneously reversing the signs of the $i$-th row of $W_0^*$ and $i$-th column of $W_1^*$ is also a solution for $i=1,...d_1$.
\end{proposition}
\begin{proof}
Let 
$$W_0^*=\left(\begin{array}{cc}
     r_1 \\
 r_2\\
 \vdots\\
 r_{d_1}
\end{array}\right),\hbox{ and } W_1^*=(c_1,c_2,\ldots,c_{d_1}),$$
where $r_i$ represents the $i$-th row of $W_0^*$ and $c_i$ represents the $i$-th column of $W_1^*$. Then system \eqref{equ:reg-grad} can be rewritten as 
\begin{equation}\label{Eq:prop5_dw1}
    \left(\begin{array}{cc}
         c_1^T  \\
         \vdots\\
         c_{d_1}^T
    \end{array}\right)((\sum_{i=1}^{d_1}c_ir_i)XX^T-YX^T)+\Lambda\circ W_0^*=0
\end{equation}
\begin{equation}\label{Eq:prop5_dw2}
    ((\sum_{i=1}^{d_1}c_ir_i)XX^T-YX^T)\left(\begin{array}{cc}
         r_1  \\
         \vdots\\
         r_{d_1}
    \end{array}\right)+\Lambda\circ W_1^*=0
\end{equation}
For \eqref{Eq:prop5_dw1}, use the property that the rows of product of two matrix are linear combinations of rows of the right side matrix, we have the following
\begin{equation}\label{eq:prop5_rows}
    c_i^T((\sum_{i=1}^{d_1}c_ir_i)XX^T-YX^T)+\Lambda(i,:)\circ r_i=0,
\end{equation}
where $\Lambda(i,:)$ represents the $i$-th row of $\Lambda$. It is immediately clear that when the signs of $c_i$ and $r_i$ changes simultaneously, \eqref{eq:prop5_rows} remains true. Thus, ~\eqref{Eq:prop5_dw1} holds. Using similar technique on the transpose of~\eqref{Eq:prop5_dw2}, we can show that \eqref{Eq:prop5_dw2} also holds. Hence, we conclude the proof. 
\end{proof}
It follows from Prop 5 that, for $H=1$ and $d_1=n$ if \eqref{equ:reg-grad} has a solutions such that all entries of $W_0^*$ and $W_1^*$ are nonzero, then it has at least $2^n$ solutions. This proof can be generalized to $H=n$ to show that given a solution $W_i^*$, when reversing the sign of the $i$-th row of $W_j^*$ and $i$-th column of $W_{j+1}$ for $i<d_j$ and $j<H+1$, it still forms a solution.

\section{Numerically Finding All the Stationary Solutions of the Deep Linear Networks}\label{sec:num_results}
Though solving systems of non-linear equations is a prohibitively difficult task, identifying the system ~\eqref{equ:reg-grad} as a system of polynomial equations, several sophisticated computational algebraic geometry techniques can be employed to find all isolated complex solutions of the system. The purely real solutions can then be trivially sorted out from the complex solutions. Symbolic methods such as the Gr\"obner basis \cite{CLO:98, CLO:07} and real algebraic geometry \cite{bpr:03} techniques can be used to solve these systems, though they may severely suffer from algorithmic complexity issues. Homotopy continuation methods have been applied to find minima and stationary points of artificial neural networks in the literature \cite{coetzee1995homotopy,485634,1555985,chow1992new,mobahi2015link,anandkumar2017homotopy}. However, though these traditional homotopy continuation methods perform well in finding multiple solutions (and often guarantee to global convergence to a solution), they do not guarantee to find all isolated solutions. In this section, we describe a sophisticated method called the numerical homotopy continuation (NPHC) \cite{SW05,Li:2003} method which guarantees to find \textit{all} complex isolated solutions of systems of multivariate polynomial equations. Then, we present our results for the deep linear networks using the NPHC method.

\subsection{The NPHC Method}
For a system of polynomial equations $\textbf{f}(\textbf{x}) = \textbf{0}$ where $\textbf{f}(\textbf{x}) = (f_1(\textbf{x}), \dots, f_n(\textbf{x}))$ and $\textbf{x} = (x_1, \dots, x_n)$, with $d_i$s being the degree of $f_i(\textbf{x})$ for all $i = 1, \dots, n$, first, one estimates an upper bound, such as the ones described in Sec.~\ref{sec:Bound}, on the number of isolated complex solutions. Then, another system $\textbf{g}(\textbf{x}) = \textbf{0}$ with $\textbf{g}(\textbf{x})= (g_1(\textbf{x}), \dots, g_n(\textbf{x}))$ is constructed such that (1) the number of complex solutions of the new system is exactly the same as the upper bound of $\textbf{f}(\textbf{x}) = \textbf{0}$, and (2) the new system is easy to solve. For the CBB, a straightforward choice for the new system $\textbf{g}(\textbf{x}) = \textbf{0}$ can be $\textbf{g}(\textbf{x}) = (x_1^{d_1} - 1, \dots, x_n^{d_n} -1 )$. For tighter upper bounds, constructing the new system may turn out to be more involved and the reader is referred to \cite{SW:05, BHSW13} for further details.

Then, a parametrized system, $\textbf{h}(\textbf{x}; t) = 0$, is formed such that $\textbf{f}(\textbf{x}) = \textbf{0}$ and $\textbf{g}(\textbf{x}) = \bf{0}$ are specific parameter points of $\textbf{h}(\textbf{x};t)$:
\begin{equation}
    \textbf{h}(\textbf{x};t) = (1-t) \textbf{f}(\textbf{x}) + \gamma t \textbf{g}(\textbf{x}) = \textbf{0}.
\end{equation}
Here, $t\in \mathbf{R}$ such that we have $\textbf{h}(\textbf{x};1) = \textbf{g}(\textbf{x}) = \textbf{0}$ and $\textbf{h}(\textbf{x};0) = \textbf{f}(\textbf{x}) = \textbf{0}$. $\gamma \in \mathbf{C}$ is a generic complex number. The system $\textbf{h}(\textbf{x};t)=0$ is also called a homotopy, more specifically, a polynomial homotopy.

Then, each complex isolated solution of $\textbf{h}(\textbf{x};1) = \textbf{g}(\textbf{x}) = \textbf{0}$, all of which are known by construction, is evolved from $t=1$ to $t=0$ using a numerical predictor-corrector method. As long as $\gamma$ is a generic complex number, all the complex isolated solutions of $\textbf{f}(\textbf{x}) = \textbf{0}$ can be reached starting from $\textbf{g}(\textbf{x}) = \textbf{0}$. Specifically, it is proven \cite{morgan1987computing} that each of such solution paths only exhibits either of the two characteristics: (1) the path converges to $t=0$ and hence a solution of $\textbf{f}(\textbf{x}) = \textbf{0}$ is found, or (2) the path diverges to infinity, i.e., the solution path is regular over $t\in (0,1]$ and yields no bifurcation, singularities, path-crossing, etc. Hence, after tracking all possible solution paths (as many as the estimated upper bound), we achieve all complex isolated solutions of $\textbf{f}(\textbf{x}) = \textbf{0}$. Moreover, the method is embarrassingly parallelizeable since each solution path can be tracked independent of each other.

\noindent \textit{Example-5:}
Using the data from Example-2, if we utilize $\lambda = 0.01$,
there are only $13$ isolated stationary points,
$5$ of which are real.  
Two are local minima that are both global minima
and the other three are saddles.

\subsection{Computational Details}\label{sec:cases_details}

In the following, we show the numerical results for the effect of changing $\Lambda$, $d_x$, $d_y$ ,$m$ and $H$ on number of isolated real solutions. For each case, we take each entry of the data matrices $X\in\bR^{d_x\times m}$ and $Y\in\bR^{d_y\times m}$ i.i.d. drawn from the Gaussian distribution with mean 0 and variance 1. We i.i.d. draw each $\lambda_i$ from $\Lambda\in[0,\delta]^{d_x d_1+\sum_{i=1}^{H-1} d_i d_{i+1}+d_H d_y}$, i.e., from the uniform distribution between 0 and $\delta$. For every case, all isolated solutions to each of 1000 samples are computed using the software \emph{Bertini} \cite{BHSW13} which is an efficient implementation of the NPHC method.
We explore how the change of any of the five variables affect the solutions of~\eqref{equ:reg_loss} for only modest size systems using \emph{Bertini} as we are restricted both in terms of computational resources as well as the number of starting solutions blowing up exponentially as a function of the system size.


\subsection{Results}
In this subsection, we provide results of solving Eqs.~\eqref{equ:reg-grad} for the cases described in section \ref{sec:cases_details}.

\subsubsection{Enumeration of Complex and Real Solutions}
First, to compare the numerical results with the upper bounds on the number of solutions discussed in section \ref{sec:Bound}, we list the bounds for various cases in Table \ref{tab:bkk1}.
\begin{table*}[ht]
    \centering
    \begin{tabular}{ccccccccccc}
        $H$ &$m$& $d_x$ &$d_y$&  $n$ &  CBB & BKK &  $\mathcal{N}_{\bC}$ & $\mathcal{N}_{DM} (H,n)$ &  max$(\mathcal{N}_{\bR})$ &  max$(I)$\\ \toprule
        1  &1 &  2 & 2 &  8 & $3^{8} = 6561$   &   1024 & 33 & 199 & 9 &2 \\ \midrule
        1  &1 &  3 & 2 & 10 & $3^{10}$         &   5184 & 33 & 592 & 9& 2\\ \midrule
        1  &1 &  4 & 2 & 12 & $3^{12}$         &  16384 & 33 & 1786 & 9 &2\\ \midrule
        1  &1 &  5 & 2 & 14 & $3^{14}$         &  40000 & 33 & 5357 &9&2\\ \midrule
        1  &1 & 10 & 2 & 24 & $3^{24}$         & 640000 &  33  & 1301759 &9&2\\ \midrule
           1  &1 &  2 & 3 & 10 & $3^{10}$              &  5184 & 73 & 592 &  9 &2 \\ \midrule
        1  &1 &  2 & 4 & 12 & $3^{12}$              & 16384 & 129 & 1786 & 9 &2 \\ \midrule
        1  &1 &  2 & 5 & 14& $3^{14}$              & 40000 & 201 & 5357  &  9 &2 \\ \midrule
               2 &1  & 2&2 & 12 &   $5^{12} = 152587890625$      & 770048  &           641  &   6250000  & 65 &3\\  \bottomrule
    \end{tabular}
    \caption{%
        Upper bounds on the number solutions for ~\eqref{equ:reg-grad} CBB, BKK and $\mathcal{N}_{DM}$ refer to the classical B\'ezout bound, BKK bound and the Didieu-Malajovich number, respectively, and are independent of the parameter values. When the network has more than one layer, $d_i=2$ for all integers $i$. $\mathcal{N}_{\bC}$ and max$(\mathcal{N}_{\bR})$ refer to the number of isolated complex to a system with generic complex parameters and the maximum number of real solutions over all the samples. max$(I)$ is the highest index of a real solution found among all the samples.}
    \label{tab:bkk1}
\end{table*}
\begin{table*}[ht]
    \centering
    \begin{tabular}{ccccccccc}
     $H$ &$m$& $d_x$ &$d_y$&  $n$ &  $\mathcal{N}_{\bC}$ & $\mathcal{N}_{DM} (H,n)$ &  max$(\mathcal{N}_{\bR})$ &  max$(I)$\\ \toprule
    
     1  &2 &  2 & 2 & 8     &              225 & 199 &29 &4\\ \midrule
          1  &2 &3&2 & 10 &              225 & 592 &29 &4\\ \midrule
        1  &3 &  2 & 2 & 8  &              225 & 199 &29 &4\\ \midrule
        
        1  &4 &2&2 & 8   &              225 & 199 &29 &4\\ \midrule
        1  &5 &2&2 & 8   &              225 & 199 &29 &4\\ \midrule
           1  &20 &2&2 & 8   &              225 & 199 &29 &4\\ \midrule
               1  &5 &3&3 & 12    &              2537 & 1786 &73 &6\\ \bottomrule

    \end{tabular}
    \caption{Computational results of $\mathcal{N}_{\bC}$, $\mathcal{N}_{\bR}$, $\mathcal{N}_{DM}$, and $\max(I)$ for the cases $m>1$. As in Table~\ref{tab:bkk1}, $\mathcal{N}_{DM}$ is independent of the choice of $\Lambda$. For the other values $\mathcal{N}_{\bC}$, mean$(\mathcal{N}_{\bR})$, and $\max(I)$, we run 1000 samples for each case with $\Lambda\in[0,1]$ and $d_i=2$ for all integers $i$.}\label{tab:bkk2}
    \end{table*}
In Table \ref{tab:bkk1}, we record the number of weights $n$, CBB, BKK, mean number of complex solutions $\mathcal{N}_{\mathbb{C}}$, the Didieu-Malajovich number of average real solutions of random polynomial cost function $\mathcal{N}_{DM}$, the maximum number of real solutions (for all $\Lambda$-values), and the maximum index among all the solutions over all samples, for various values of $H$, $m$, $d_x$ and $d_y$ while fixing $d_1 =\dots= d_H=2$. The CBB grows exponentially with the number of variables. Though the BKK count grows rapidly as well, it is significantly smaller than CBB. However, the average number of complex solutions computed using the NPHC method is even smaller compared to these two bounds. Moreover, the maximum number of real solution is also smaller than Didieu-Malajovich number. Both these observations yield that our gradient system is highly sparse and structured compared to that of the dense polynomial cost functions \eqref{eq:random_poly}.

In Table \ref{tab:bkk2}, we record numerical results for mean number of complex solutions, $\mathcal{N}_{\mathbf{C}}$, maximum number of real solutions out of all samples max$(\mathcal{N}_{\mathbf{R}})$ and the maximum index, max$(I)$ out of all real solutions of all samples. 
Note that $m=1$ is a pathological case as it refers to only one data point case, but it still provides nonlinearity to the gradient equations to have non-trivial solutions. Moreover, the value of $m$ does not change the order of the polynomials but only the monomial structure of the polynomials. When $m=1$, the matrix $XX^T$ in \eqref{equ:reg-grad} is singular and of rank 1, which implies there are additional structure. For $m>1$, $XX^T$ is nonsingular with probability 1, the polynomial system has the very same structure yielding a constant number of complex solutions for generic values of $X$ and $Y$. 

\subsubsection{Distribution of Number of Real Solutions}
To see the impact of the regulation term $\Lambda$, we change the maximum value of the interval on which $\Lambda$ is uniformly distributed. In Figure \ref{fig:rs_lbd}, we show how the distribution of $\mathcal{N}_{\bR}$ changes as the range of $\Lambda$ changes: the mean number of reals solutions decreases as the range of $\Lambda$ increases, which yields the phenomenon of topology trivialization \cite{Kastner:2011zz,Mehta:2012qr,fyodorov2013high,Mehta:2014xya,chaudhari2015trivializing}. As $\Lambda$ values increase beyond 1, there are more samples with no real solution. Also, as $\Lambda$s approach to zero, the mean number of real solutions becomes relatively stable, though the condition number of the real solutions begin to increase which is expected because system \eqref{equ:reg-grad} tends to be the unregularized case.  

Figure \ref{fig:rs_xym} demonstrates the impact on average number of real solutions of $d_x$, $d_y$ and $m$. It yields that increase of any of the three parameters increases the mean number of real solution. Combining Figure \ref{fig:rs_xym} and Table \ref{tab:bkk1}, one notices that the more data points there are, the more real solutions to \eqref{equ:reg-grad}, on an average, there are.

\begin{figure}[!htbp]
    \centering
    \includegraphics[width=2in]{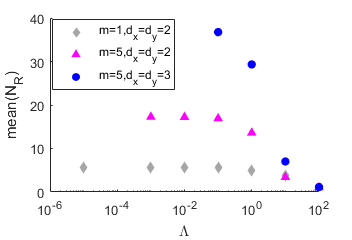}
    \caption{The mean of $\mathcal{N}_{\bR}$ as a function of range of $\Lambda$ values, for $H=1$ and $d_1=2$.}
    \label{fig:rs_lbd}
\end{figure}
\begin{figure}
    \centering
    \includegraphics[width=2in]{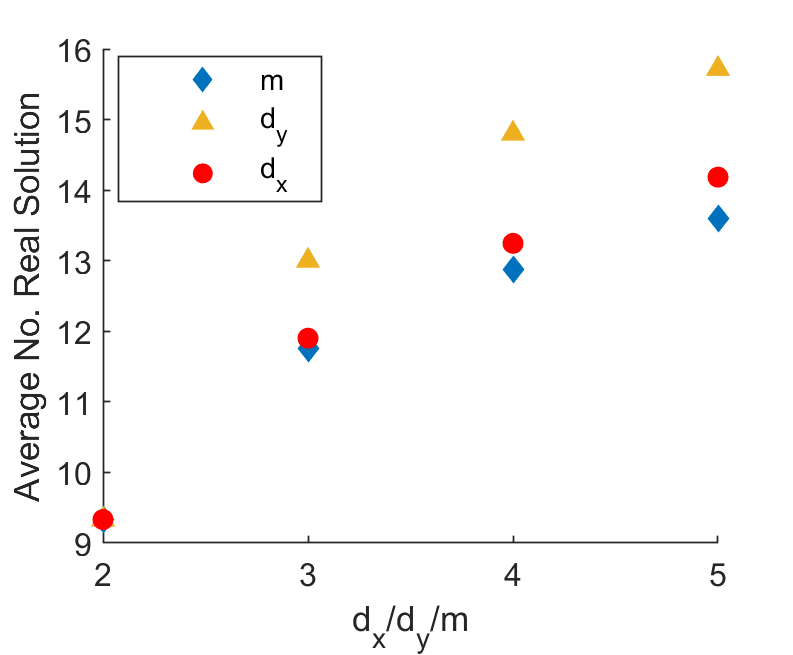}
    \caption{The mean of $\mathcal{N}_{\bR}$ for different value of $d_x$ (hexagon), $d_y$ (triangles) and $m$ (diamonds). For the cases represented by diamonds, $H=1$, and $d_x=d_y=d_1=2$. For the cases represented by triangles, $H=1$, $m=2$ and $d_x=d_1=2$. For the cases represented by circles, $H=1$, $m=1$ and $d_x=d_1=2$. The range of $\Lambda$ is $[0,1]$ for all cases. }
    \label{fig:rs_xym}
\end{figure}
\subsubsection{Index-resolved Number of Real Solutions}
For each case and each sample, we compute the ratio of number of real solution with index $I$, $\mathcal{N}_I$, to the total number of real solutions. Then we calculate the index distribution by taking the mean of the ratios for 1000 samples for each case. For the samples for which there are no solutions, we set the ratio to be 0.

Figure \ref{fig:index_dx} and \ref{fig:index_m} demonstrate the index distribution for different $d_x$, range of $\Lambda$, and $m$ respectively. From table~\ref{tab:bkk2} and right figure of Figure~\ref{fig:index_dx}, we notice that, for the cases $H=1$, $m=2$, and $d_y=d_1=2$, even though the number of variable increases from 8 to 14 as $d_x$ increasing from 2 to 5, the number of isolated solution remains the same. It also shows that when $H=1$, $m=2$, $d_y=d_1=2$, and $\Lambda\in[0,1]$, the highest index is 4 and the probability of a solutions to~\eqref{equ:reg-grad} is not an extrema of $\mathcal{L}^{\Lambda}$ increases. 

The left figure of Figure~\ref{fig:index_dx} reveals that as the range of regulation term $\Lambda$ approaches 0, the index distribution reach an equilibrium. 
Figure~\ref{fig:index_m} informs us that, for cases where $H=1$ and $d_x=d_y=d_1=2$, the highest solution index is 4 and the peak frequency is always reached by solution with index 2. Similar as in the cases with different $d_x$, the probability of a solution to~\eqref{equ:reg-grad} is a saddle point of $\mathcal{L}^{\Lambda}$ increases sub-linearly as the number of data points increases. 

\begin{figure}
    \centering
     \includegraphics[width=3.5in]{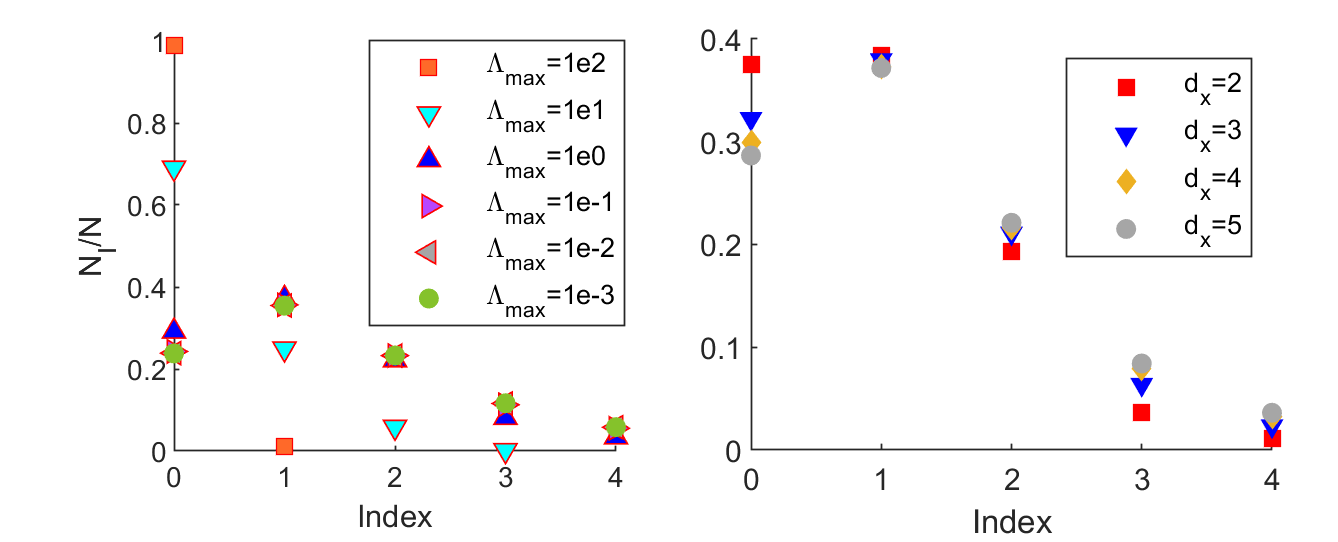}
    \caption{The index distribution for different $d_x$ and range of $\Lambda$. For the figure on the left side, the parameters are set to be $H=1$, $m=5$, $d_x=d_y=d_1=2$, and $\Lambda\in[0,\Lambda_{\max}]$. For the figure on the right side, the other parameters are set to be $H=1$, $m=2$, $d_y=d_1=2$, and $\Lambda\in[0,1]$.}
    \label{fig:index_dx}
\end{figure}

\begin{figure}
    \centering
    \includegraphics[width=2in]{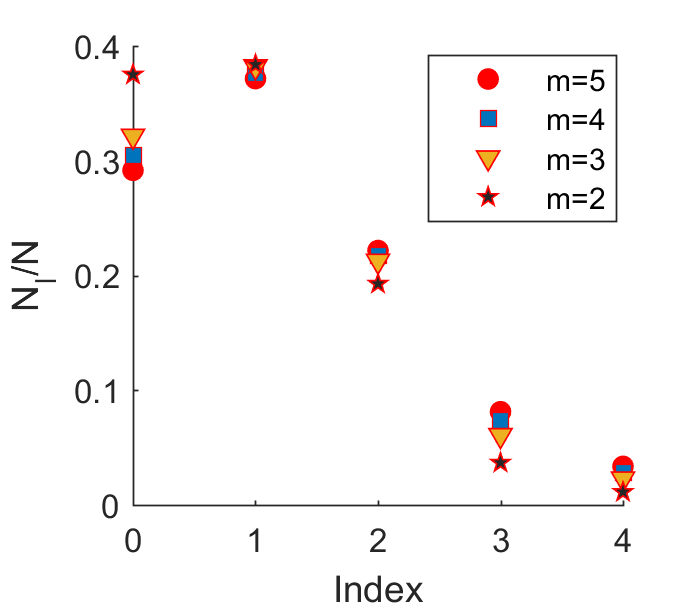}
    \caption{The index distribution for different $m$.The other parameters are set to be $H=1$, $d_x=2$, $d_y=2$ and $\Lambda\in[0,1]$.}
    \label{fig:index_m}
\end{figure}

\subsubsection{Minima}
In the following, we take a closer look at the structure of all the real solutions for each sample. It is straightforward to verify that the configuration with all weights being zero is always a solution of Eqs.~\eqref{equ:reg-grad}. For the cases where $H=1$ and $m=1$, we have total of $13,000$ samples consisting of 13 different scenarios listed in Table~\ref{tab:bkk1}, \ref{tab:bkk2} and Figure~\ref{fig:rs_lbd}. For this particular case, all local minima are global minima, and the absolute values of all the local minima are the same, i.e., all the minima are symmetrically related to each other.

For the cases where $H>1$ or $m>1$, we have total of 15000 samples consisting of 15 different scenarios as listed in Table~\ref{tab:bkk1}, \ref{tab:bkk2} and Figure~\ref{fig:rs_lbd}. Here, we observe instances where there exist local minima which are not global minima. One such instance is given in Table~\ref{tab:sol_lneqg} and Table~\ref{tab:param_lneqg}.

For the cases where $H=1$ and $m=5$, we notice that all sample runs with $\Lambda\in[0,100]$ exhibit all local minima are global minima, though, combining with the observation from Figure~\ref{fig:rs_lbd}, this may only be an artifact of the topology trivialization.

\begin{table*}[ht]
  \centering
    \begin{tabular}{rrrrrrrrrrrrr}
        $w^0_{11}$& $w^0_{21}$ &$w^0_{12}$&$w^0_{22}$& $w^0_{13}$ &$w^0_{23}$  & $w^1_{11}$ & $w^1_{21}$ &$w^1_{31}$ & $w^1_{12}$ & $w^1_{22}$ &$w^1_{32}$ & $\mathcal{L}^{\Lambda}$\\ \toprule
    
   0.42959 & 0.36758 & 0.30899 & -0.10019 & -0.01419 & -0.23650 & -0.50336 & 0.33655 & -0.01843 & -0.11969 & -0.14925 & 0.54928 & 7.13717 \\ \midrule
    -0.42959 & 0.36758 & -0.30899 & -0.10019 & 0.01419 & -0.23650 & 0.50336 & -0.33655 & 0.01843 & -0.11969 & -0.14925 & 0.54928 & 7.13717 \\\midrule
    -0.42959 & -0.36758 & -0.30899 & 0.10019 & 0.01419 & 0.23650 & 0.50336 & -0.33655 & 0.01843 & 0.11969 & 0.14925 & -0.54928 & 7.13717 \\\midrule
    0.42959 & -0.36758 & 0.30899 & 0.10019 & -0.01419 & 0.23650 & -0.50336 & 0.33655 & -0.01843 & 0.11969 & 0.14925 & -0.54928 & 7.13717 \\\midrule
    0.54286 & -0.05927 & 0.22411 & -0.05389 & -0.04306 & 0.26254 & -0.51936 & 0.16009 & 0.25030 & 0.05058 & -0.17838 & -0.07580 & 7.16775 \\\midrule
    0.54286 & 0.05927 & 0.22411 & 0.05389 & -0.04306 & -0.26254 & -0.51936 & 0.16009 & 0.25030 & -0.05058 & 0.17838 & 0.07580 & 7.16775 \\\midrule
    -0.54286 & -0.05927 & -0.22411 & -0.05389 & 0.04306 & 0.26254 & 0.51936 & -0.16009 & -0.25030 & 0.05058 & -0.17838 & -0.07580 & 7.16775 \\\midrule
    -0.54286 & -0.05927 & -0.22411 & -0.05389 & 0.04306 & 0.26254 & 0.51936 & -0.16009 & -0.25030 & 0.05058 & -0.17838 & -0.07580 & 7.16775 \\\bottomrule
    \end{tabular}%
     \caption{An instance where some local minimum are not global minimum. All real solutions with hessian index 0 are listed. The case settings are $H=1$,$m=5$, $d_x=d_y=3$, $d_1=2$, and $\Lambda\in[0,1]$. }
 
  \label{tab:sol_lneqg}%
\end{table*}%

\begin{table*}[htbp]

  \centering
      \begin{tabular}{rrrrrrrrr}
  $x_{11}$ & $x_{21}$& $x_{31}$ & $x_{12}$& $x_{22}$& $x_{32}$& $x_{13}$& $x_{23}$& $x_{33}$\\\midrule
     -0.1297 & -1.0135 & 0.2523 & 0.5236 & -1.4616 & 1.8664 & -2.1491 & -1.6352 & 1.2240 \\\bottomrule
    $x_{14}$&$x_{24}$&$x_{34}$&$x_{15}$&$x_{25}$&$x_{35}$&$y_{11}$&$y_{21}$ &$y_{31}$\\\midrule
     0.3252 & -0.4289 & 0.0116 & 0.7313 & -0.8680 & 0.9282 & 0.6973 & -0.0452 & 0.1912 \\\bottomrule
     $y_{12}$&$y_{22}$ &$y_{32}$&$y_{13}$&$y_{23}$ &$y_{33}$&$y_{14}$&$y_{24}$ &$y_{34}$\\\midrule
      -0.6288 & -0.8566 & -0.3887 & 1.0285 & -0.2397 & -0.4516 & -0.9793 & -1.1334 & 0.0221 \\\bottomrule
      $y_{15}$&$y_{25}$ &$y_{35}$&       &       &       &       &       &  \\\midrule
       1.0402 & 1.2315 & 0.5602 &       &       &       &       &       &  \\\bottomrule
    \end{tabular}\vspace{0.3in}
    
    \begin{tabular}{rrrrrrrrrrrr}
    $\Lambda^0_{11}$&$\Lambda^0_{21}$&$\Lambda^0_{12}$ &$\Lambda^0_{22}$&$\Lambda^0_{13}$& $\Lambda^0_{23}$  & $\Lambda^1_{11}$ & $\Lambda^1_{21}$   & $\Lambda^1_{31}$ &$\Lambda^1_{12}$& $\Lambda^1_{22}$&$\Lambda^1_{32}$\\\toprule 
  0.383 & 0.298 & 0.6917 & 0.8805 & 0.9245 & 0.0813 & 0.4827 & 0.1283 & 0.2529 & 0.884 & 0.1963 & 0.1214 \\\bottomrule    \end{tabular}%
    \caption{The parameters used to generate the instances with solutions listed in Table~\ref{tab:sol_lneqg}.}

  \label{tab:param_lneqg}%
\end{table*}%

\subsubsection{Loss Function at Real Solutions}
To see how the range of $\Lambda$ affect the loss function value, we compute the global minimum of each sample and plot the minimum, mean and maximum of 1000 samples of global minimum for different $\Lambda$ range. In Figure~ \ref{fig:costf}, the $\log\log$ plot of the minimum, mean, and maximum of 1000 samples for different $\Lambda$ are shown. We observe that as the range of $\Lambda$ approaches zero, the mean and minimum of the global minimum approach to nonzero constants. This implies that, for a generic case with 5 data points and two one hidden layer, there is no parameter values of $W_0$ and $W_1$ such that the loss function achieves global minimum of 0, i.e., the zero training error minima.
\begin{figure}
    \centering
    \includegraphics[width=3in]{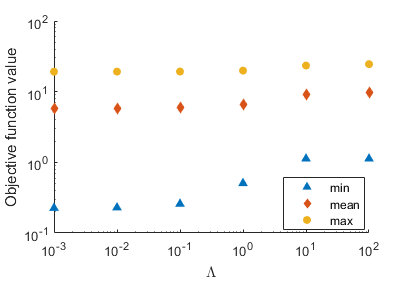}
    \caption{The minimum, mean and maximum of global minimum loss function value at real solutions of 1000 samples with different range of $\Lambda$. The other parameters are set as $H=1$, $M=5$, $d_1=d_x=d_y=2$.}
    \label{fig:costf}
\end{figure}




\section{Conclusions and Discussion}\label{sec:conclusion}
Understanding non-convexities of the optimization problems in deep learning and their implications in learning are an active area of research. Deep linear networks have served as an ideal test ground of ideas as they qualitatively captures certain features of the deep non-linear networks yet simple enough for analytical and numerical investigations. In the present paper, we have initiated an ambitious plan to understand the loss landscapes of deep networks from the algebraic geometry point of view. Our approach is to provide practicable results from algebraic geometry rather than abstract ones, by invoking computational and numerical algebraic geometry methods. 

\noindent\textit{Algebraic Geometry Interpretation:-}
In the present paper, after reviewing the existing results on the deep learning loss surfaces as well as for deep linear loss surfaces, we observed that the system of gradient equations of the deep linear networks is an algebraic system and argued that by complexifying the equations brings the problem of solving this system into the complex algebraic geometry domain. In turn we can utilize many of the mature results and methods from algebraic geometry to gain insights into the optimization landscapes of these systems.

We emphasis that the algebraic geometric interpretation of gradient equations is not restricted only to the deep linear networks: The classes of deep nonlinear networks which obviously fall under the algebraic geometry paradigm are deep polynomial networks and deep complex networks. While any other activation functions can be approximated by polynomials of finite degrees, the gradient systems for most of the contemporary activation functions used for deep nonlinear networks in practice such as hyperbolic tangent, sigmoid, rectified linear units (ReLUs), leaky ReLUs, Heaviside, etc. activation functions are, or can be transformed to, form algebraic systems. Hence, the results and methods can also be applied, after appropriate modifications, to investigate loss landscapes of deep nonlinear networks.

\noindent\textit{Flat Stationary Points and their Regularization:-}
We then reviewed the current understanding of the ''flat'' minima in deep learning and provided a distinction among different definitions of ''flat'' minima and other stationary points. In particular, a flat stationary point in our case is a connected component in the weight space such that each of the points on this component are solutions of the gradient equations and that the loss function remains strictly constant when evaluated over the whole component. Such flat stationary points also called positive-dimensional solutions where the dimension refers to the (real or complex) dimension of the component. Such a flat minimum over the real space is distinct from an isolated stationary point in the real space, though the hessian matrices evaluated at both of which are singular.

For deep linear networks, in the present paper, we showed that there do exist positive-dimensional components when no regularization is used. In the existing literature, the deep linear networks are shown to posses no local minima which are not global minima. Our results then yield that the loss surface of unregularized deep linear network consists of \textit{minima lakes} each of which are at the same level as the global minimum. In fact, the landscape also consists of \textit{stationary point lakes} with the hessian matrix having higher index at these solutions.

Then, using the generalized Sard's theorem, we showed that when an extension of $L_2$-regularization is added to the loss function, all (complex and real) stationary points become isolated, i.e., no flat stationary points exist. In addition, this regularization also removes isolated singular solutions. 

Since the stochastic gradient (SGD) descent method and its variants rely only on the first order (gradients) information while searching for a minimum, they have to pass through saddle points of higher index. The number of saddle points of higher index is usually exponentially more than the number of minima in such a high dimensional and nonllinear loss landscapes. In addition, if there are flat saddle points present in the system, the SGD may encounter further issues such as the computation getting stuck on the flat saddle point, in turn performance plateauing for many epochs. Recently, a few attempts have been made to device methods that escape from wide minima in the absence of singular solutions (and in presence of singular solutions in limited cases) \cite{dauphin2014identifying,pmlr-v40-Ge15,nesterov2006cubic,anandkumar2016efficient,2017arXiv171007406L,panageas2016gradient}. An alternative way to evade the singular solutions (both flat and degenerate) may be to use the proposed regularization which eliminates flat stationary points and minima right from the beginning of the SGD computation and hence the wide minima escaping methods can then be applied to achieve better training.

The existence and implications of flat minima have been discussed in the existing literature. In particular, it has been argued that networks trained on flat minima generalize more than when trained on sharp minima. On the other hand, it is also argued that flat minima can be easily converted to sharp minima using a reparametrization. Our results confer the former argument, though in the paradigm of the definition of flatness in the algebraic geometry sense. We also argue that since in general the loss landscapes quantitatively (and in some cases even qualitatively) changes with respect to data, unless the ''flatness'' (however defined) of the minima is an invariant of the data, the existence of flat stationary points may not be crucial for the generalization ability of the network. On the other hand, the existence of an invariance of flatness of minima and saddle points, if proven, may turn out to be crucial in understanding the generalization propoerties.

It should be noted that the existence of flat stationary points directly corresponds to continuous symmetries in the system. Various ways to break these continuous symmetries have been investigated in the literature \cite{Wales03}. The generalized $L_2$ regularization term essentially perturbs the system to leave only isolated solutions in the system. e.g., in \cite{orhan2017skip}, it is argued that skip connections in neural networks eliminate singularities as it removes certain symmetries from the system. It may be interesting to study relation between the generalized $L_2$ regularization and skipped connections. One can also project the constant zero modes of the hessian in the computation \cite{Wales03}. From the computational point of view though, the generalized $L_2$ regularization approach may be the most straightforward way to implement in the current deep learning suites.

\noindent\textit{Upper Bounds on the Number of Stationary Points and Numerical Results:-}
Once all the flat stationary points are removed from the gradient equations, the next question we addressed is how many isolated stationary points are there? When the gradient equations are treated as defined over complex space, there are many upper bounds, such as the CBB and BKK bounds, on the number of isolated complex solutions for systems of polynomial equations available in the literature which we can employ to gain insight into our systems. For the deep linear networks, the CBB and BKK for modest size networks are given in Table \ref{tab:bkk1}. 

Using these upper bounds, we employed a numerical algebraic geometry method called the numerical polynomial homotopy continuation method which guarantees to find all isolated complex solutions of such polynomial systems. In our experiments, we generated data matrices $X$ and $Y$ by drawing each of their entries independently drawn from the Gaussian distribution with mean 0 and variance 1 and $\lambda_i$s from uniform distributions between $[0,1], [0, 0.01], \dots, [0,0.00001]$ to investigate the effect of different magnitudes of $\lambda_i$s. The average number of complex and real solutions over all samples are compared with the CBB and BKK in Table \ref{tab:bkk1}. We compared these bounds with the available analytical result on the average number of real stationary points of random polynomial cost function. The average number of complex and real solutions of these systems are orders of magnitudes smaller than these bounds confirming that the deep linear systems are very \textit{sparse}. This conclusion \textit{may or may not} necessarily extend to deep nonlinear network as the structure of the corresponding polynomials may different. 

We showed that the average number of real solutions reduces as $\lambda_i$s vanishes, the phenomenon called topology trivialization \cite{Kastner:2011zz,Mehta:2012qr,fyodorov2013high,Mehta:2014xya,chaudhari2015trivializing}, and singular solutions appear as the system tends to the unregularized case where the solutions are flat.

We sorted the stationary points in terms of index (number of negative eigenvalues) of hessian matrix and showed that for some samples, \textit{there are indeed local minima which are not global minima} contrary to the available results in the unregularized case. This result is a first for the complete deep linear networks in the regularized case (in fact, Ref.~\cite{taghvaei2017regularization} is the only result available for the linear networks with non-zero regularization in a restricted case). There exist many discrete symmetries among solutions, i.e., the value of the loss function at the symmetrically related solutions is equal. We also notice that the stationary points with higher index are rare, which may be due the linearity of the activation functions but may not necessarily be a phenomenon for the nonlinear activation functions.

Investigating the loss landscapes when $X$ and $Y$ are correlated, instead of choosing their values from random distributions, may exhibit interesting characteristics of the optimization landscapes as well as the working of the deep learning to fit a given data set. Extending the algebraic geometry interpretation to deep nonlinear networks will shed further novel insights into the optimization landscapes of these models. Computational implications of the proposed regularization approach specially together with the saddles escaping methods will be an important breakthrough on these theoretical insights.

\bibliographystyle{ieeetr}
\bibliography{bibliography_NPHC_NAG}

\end{document}